\documentclass[10pt,twocolumn,letterpaper]{article}

\usepackage{iccv}
\usepackage{times}
\usepackage{epsfig}
\usepackage{graphicx}
\usepackage{amsmath}
\usepackage{amssymb}

\usepackage[latin9]{inputenc}
\usepackage{amsmath}
\usepackage{amssymb}
\usepackage{wasysym}
\usepackage{amssymb}
\usepackage{amsthm}
\usepackage{color}
\usepackage{amsmath}
\usepackage{amsmath}
\usepackage{amssymb}
\usepackage{amsthm}
\usepackage{amsmath}
\usepackage{comment}
\usepackage{graphicx}
\usepackage{titling}
\pretitle{\begin{center}\fontsize{15bp}{15bp}\selectfont}
    \posttitle{\par\end{center}}
\preauthor{\begin{center}\fontsize{12bp}{12bp}\selectfont}
    \postauthor{\par\end{center}\vspace{24bp}}
\predate{}
\date{}
\postdate{}

\usepackage{float}
\usepackage{multirow}
\usepackage{epsfig}
\usepackage{graphicx}
\usepackage[english]{babel}
\newtheorem{theorem}{Theorem}[]
\newtheorem{corollary}{Corollary}[]
\newtheorem{lemma}[theorem]{Lemma} 

\newtheorem{definition}{Definition}[]
\newcommand{\Real}{\mathbb R}
\newcommand{\norm}[1]{\left\lVert #1 \right\rVert}

\PassOptionsToPackage{normalem}{ulem}
\usepackage{ulem}
\usepackage{babel}

\definecolor{mypink1}{rgb}{0.33, 0.76, 0.66}

\newcommand{\rb}[1]{\textcolor{blue}{[Ronen: #1]}}

\newcommand{\ag}[1]{\textcolor{red}{[ #1]}}

\usepackage[pagebackref=true,breaklinks=true,letterpaper=true,colorlinks,bookmarks=false]{hyperref}
\usepackage[titlenumbered,ruled]{algorithm2e}
 \iccvfinalcopy 

\def\iccvPaperID{1621} 

\ificcvfinal\fi
\begin{document}

\title{Algebraic Characterization of Essential Matrices and Their Averaging\\
in Multiview Settings}
\author{Yoni Kasten* \hspace{1cm} Amnon Geifman* \hspace{1cm} Meirav Galun  \hspace{1cm}Ronen Basri \\
Weizmann Institute of Science\\
{\tt\small  \{yoni.kasten,amnon.geifman,meirav.galun,ronen.basri\}@weizmann.ac.il}
}
\maketitle

\begin{abstract}
   Essential matrix averaging, i.e., the task of recovering camera locations and orientations in calibrated, multiview settings, is a first step in global approaches to Euclidean structure from motion. A common approach to essential matrix averaging is to separately solve for camera orientations and subsequently for camera positions. This paper presents a novel approach that solves  simultaneously for both camera orientations and positions. We offer a complete characterization of the algebraic conditions that enable a unique Euclidean reconstruction of $n$ cameras from a collection of $(^n_2)$ essential matrices. We next use these conditions to formulate essential matrix averaging as a constrained optimization problem, allowing us to recover a consistent set of essential matrices given a (possibly partial) set of measured essential matrices computed independently for pairs of images. We finally use the recovered essential matrices to determine the global positions and orientations of the $n$ cameras. We test our method on common SfM datasets, demonstrating high accuracy while maintaining efficiency and robustness, compared to existing methods.
\end{abstract}

\makeatletter
\def\blfootnote{\xdef\@thefnmark{}\@footnotetext}
\makeatother
\blfootnote{*Equal contributors}


\section{Introduction}
What algebraic conditions make a collection of $n \choose 2$ essential matrices consistent, in the sense that there exist $n$ Euclidean camera matrices that produce them? This fundamental question has not yet been answered in the literature. It is well known that ${3 \choose 2} = 3$ fundamental matrices are consistent if, and only if, the epipole of the third view is transferred correctly between each pair of views, i.e., for every $1 \le i,j,k \le 3,$ $\mathbf e_{ik}^TF_{ij}\mathbf e_{jk}=0$. Recent work~\cite{kasten2018gpsfm} presented a set of sufficient and necessary algebraic conditions that make $n \choose 2$ fundamental matrices in general position consistent. One could expect that essential matrices that fulfill those same conditions would be consistent with respect to Euclidean camera matrices.  However, these conditions are not sufficient and can be contradicted by a counter example, see one such construction in the supplementary material. 

Establishing consistency constraints for essential matrices is an important step toward producing essential matrix averaging algorithms. Given $n$ images $I_1,..,I_n$, a common approach for global Structure from Motion (SfM)  begins by robustly estimating essential matrices between pairs of views, $\{E_{ij}\}$, from which an estimate of the relative pairwise rotations $\{R_{ij}\}$ and translations $\{\mathbf{t}_{ij}\}$ are extracted.  Motion averaging then is performed typically in two steps: first the  absolute camera orientations $\{R_i\}$ are solved  by averaging the relative rotations. Then, using  the relative translations and the recovered absolute orientations, the absolute camera positions $\{\mathbf t_i\}$ are recovered. Finally, the obtained solution is refined by bundle adjustment.


Our goal in this paper is to establish a complete set of necessary and sufficient conditions for the consistency of essential matrices and to use these conditions to formulate  a one-step algorithm for averaging essential matrices. To achieve this goal we investigate an object called the $n$-views essential matrix, which is obtained by stacking the $n \choose 2$ essential matrices into a $3n\times3n$ matrix whose $i,j$'th $3 \times 3$ block is the essential matrix $E_{ij}$ relating the $i$'th and the $j$'th frames.   We prove that, in addition to projective consistency (introduced in~\cite{kasten2018gpsfm}), this matrix must have three pairs of eigenvalues each of the same magnitude but opposite signs, and its eigenvectors directly encode camera parameters. 

We use these results to introduce the first (to the best of our knowledge) essential matrix averaging algorithm. Given a noisy estimate of a subset of $n \choose 2$ essential matrices, our algorithm seeks to find the nearest consistent set of essential matrices. We formulate this problem as constrained optimization and solve it using ADMM. We then incorporate this algorithm in a global SfM pipeline and  evaluate our pipeline on the datasets of \cite{wilson2014robust}, showing superior accuracies relative to state of the art methods on almost all image collections while also maintaining efficiency. 



\section{Related work}

Approaches for Euclidean motion averaging can be divided into two main categories: Incremental methods \cite{klopschitz2010robust,snavely2008modeling,agarwal2009building,kasten2019resultant,wu2013towards} begin with a small subset of frames and produce an initial reconstruction. The rest of cameras are then used sequentially for reconstruction. These methods are very successful and quite robust. However, they have to apply bundle adjustment refinement  at every step to prevent camera drift. Consequently, these method are computationally demanding when applied to  large data sets. 

Global methods \cite{arie2012global,wilson2014robust,goldstein2016shapefit,jiang2013global,ozyesil2015robust}, in contrast, recover the motion parameters simultaneously for all the frames.  Typical global SfM pipelines proceed by applying a camera  orientation solver, followed by a location solver. \\
\uline{\textbf{Global orientation solvers}}   \cite{arie2012global,martinec2007robust,tron2009distributed,hartley2013rotation,chatterjee2018robust} solve for the absolute orientations of the cameras given relative rotation measurements between pairs of images. \cite{arie2012global,martinec2007robust} derive  closed form   solutions that minimize a least squares objective constructed from the pairwise relative orientations. These methods are very efficient but due to the relaxed orthonormality requirement, the result is  usually suboptimal. Other methods \cite{tron2009distributed,hartley2013rotation,chatterjee2018robust} utilize the Lie algebra structure of the rotation group to perform rotation averaging in $SO(3)$. These methods, however, often converge to  local minima. Recently, \cite{eriksson2018rotation} has used the strong duality principle to find the global minimum under certain conditions.\\
\uline{\textbf{Global location solvers}}  \cite{arie2012global,wilson2014robust,ozyesil2015robust} assume known  camera orientations  and solve for the absolute positions of the cameras by using the noisy relative translations.  \cite{arie2012global} uses  point correspondences to find a least squares solution for the absolute positions. \cite{wilson2014robust} formulates a highly non convex objective and solves for the absolute translations utilizing the Levenberg-Marquet algorithm with random initialization. \cite{ozyesil2015robust} uses a similar objective while formulating a convex relaxation. \cite{cui2015linear} uses linear global method that minimizes geometric error in triplet of views while considering feature tracks.   All the aforementioned methods are highly dependent on accurate estimation of the absolute rotations of the cameras, which result from a rotation averaging method.\\
\uline{\textbf{Integrative Methods}}: \cite{sweeney2015optimizing} formulates the problem of fundamental and essential averaging as global optimization through minimizing the epipolar transfer error. While their method manages to improve the consistency of viewing graphs, it is unable  to generate a consistent reconstruction, and so it requires postprocessing steps of both rotation and translation averaging. \cite{sengupta2017new} has introduced the concept of a multi-view essential/fundamental matrix. Their work, however, established only a partial list of constraints. Moreover, their use of a complicated non-convex objective allowed them to  only  refine a complete reconstruction. \cite{kasten2018gpsfm} has introduced a complete set of consistency conditions  for fundamental matrices. They   formulate a robust optimization objective and demonstrate state of the art projective reconstructions. Their method, however, is limited to projective settings and is inapplicable to calibrated settings, i.e., for Euclidean reconstruction. \cite{tron2009distributed} suggested a method that optimizes first for camera positions  and then for  their orientations, and as a post processing   simultaneously optimizes for both. However, this method is sensitive to outliers.  Recent work explored the properties of the manifold of  essential matrices \cite{tron2017space}. Their characterization, however, is   suitable only for a single essential matrix and not for general multiview settings. Finally, \cite{aholt2014ideal,heyden1997algebraic} explore general algebraic properties of multi-view settings.

Our paper extends the work of \cite{kasten2018gpsfm,sengupta2017new} by introducing a complete set of necessary and sufficient conditions for consistency of multiview essential matrices and by proposing an efficient and robust optimization algorithm for essential matrix averaging that incorporates these conditions.

\section{Theory }

Let $I_1, ..., I_n$ denote a collection of $n$ images of a static scene captured respectively by cameras $P_1,...,P_n$. Each camera $P_i$ is represented by a $3 \times 4$ matrix $P_i=K_i R_i^T[I,-t_i]$ where $K_{i}$ is a $3\times 3$ calibration matrix, $\bf{t}_i\in \mathbb{R}^3$ and $R_i \in SO(3)$ denote the position and orientation of $P_i$, respectively, in some global coordinate system.   We further denote $V_i=K_i^{-T}R_i^{T}$, so $P_i=V_i^{-T}[I,-{\bf t}_i]$.  Consequently, let $\bold{X}=(X,Y,Z)^T$ be a scene point in the global coordinate system. Its projection onto $I_i$ is given by ${\bf x}_i = \bold{X}_{i} / Z_i$, where $\bold{X}_{i}=(X_i, Y_i, Z_i)^T = K_{i}R_{i}^{T}(\bold{X}-{\bf t}_{i})$.  
 

We denote the fundamental matrix and the essential matrix between images $I_i$ and $I_j$ by $F_{ij}$ and $E_{ij}$ respectively. It was shown in \cite{arie2012global} that $E_{ij}$ and $F_{ij}$ can be written as
$$E_{ij} = R_i^T (T_i - T_j) R_j$$
$$F_{ij} = K_i^{-T} E_{ij} K_j^{-1}=V_i(T_i-T_j)V_j^T$$
where $T_i=[{\bf t}_i]_{\times}$.

Throughout this paper we assume that all calibration matrices are known, so our work deals with solving the problem of Euclidean SfM.

The derivations in this paper adopt the definitions   ``$n$-view  fundamental matrix" and ``consistent $n$-view fundamental matrix" from \cite{kasten2018gpsfm}. We first repeat these definitions, for the sake of clarity, and then define analogous definitions for the calibrated case.  In the definitions below we  denote the space of all the $3n \times 3n$ symmetric matrices by $\mathbb{S}^{3n}$.    
\begin{definition}   \label{def:mv_fundamental} 
A matrix $F \in \mathbb{S}^{3n}$, whose $3 \times 3$ blocks are denoted by $F_{ij}$, is called an {\bf n-view  fundamental matrix} if  $ \forall  i\neq j \in \{1,..., n\},$ $rank(F_{ij})=2$ and $\forall  i $  $F_{ii}=0$.
\end{definition}

\begin{definition}   \label{def:consistent_mv_fundamental} 
An $n$-view  fundamental matrix $F$ is called {\bf consistent}  if there exist camera matrices $P_1,...,P_n$ of the form $P_{i}=V_{i}^{-T}[I,\mathbf{t}_{i}]$ such that $F_{ij}=V_{i}([\mathbf{t}_{i}]_{\times}-[\mathbf{t}_{j}]_{\times})V_{j}^{T}$.
\end{definition}

\begin{definition}  \label{def:mv_essential} 
A matrix $E \in \mathbb{S}^{3n}$, whose $3 \times 3$ blocks are denoted by $E_{ij}$, is called an {\bf n-view essential  matrix} if   $\forall  i\neq j \in \{1\dots n\}$ , $rank(E_{ij})=2$, the two singular values of $E_{ij}$ are equal, and $\forall  i $ $  E_{ii}=0$. We denote the set of all such matrices by ${\cal E}$.
\end{definition}

\begin{definition}

 An $n$-view essential  matrix $E$ is called {\bf consistent} if there exist $n$ rotation matrices $\{R_i\}_{i=1}^{n}$ and $n$  vectors $\{{\bf t}_i\}_{i=1}^{n}$  such that $E_{ij}= R_{i}^{T}([\mathbf{t}_{i}]_{\times}-[\mathbf{t}_{j}]_{\times})R_{j}$. \label{def:consistent_mv_essential}
\end{definition}

Note that any (consistent) $n$-view essential matrix is also a (consistent) $n$-view fundamental matrix. In \cite{kasten2018gpsfm}   necessary
and  sufficient conditions for the consistency of the $n$-view fundamental matrix were proved. The main theoretical contribution of \cite{kasten2018gpsfm} is summarized in Theorem~\ref{thm:consistent_direct}. For  the consistency of $n$-view essential matrix, a partial set of necessary conditions were derived in \cite{sengupta2017new}. Those are summarized below in Theorem \ref{thm:necessary_essential}.

\begin{theorem}\label{thm:consistent_direct} 
 An n-view fundamental matrix $F$ is consistent with a set of $n$ cameras whose centers are not all collinear if, and only if, the following conditions hold: 
\begin{enumerate}
\item $Rank(F)=6$ and $F$ has exactly 3 positive and 3 negative eigenvalues.
\item $Rank(F_i) = 3$ for all $i =1, ..., n$, where $F_i$  is the $3 \times 3n$  $i^{th}$ block row of $F$.\end{enumerate}
\end{theorem}

\begin{theorem}\label{thm:necessary_essential} Let $E$ be a consistent $n$-view essential matrix, associated with rotation matrices $\{R_i\}_{i=1}^n$ and camera centers $\{{\bf t}_i\}_{i=1}^n$.  $E$ satisfies the following conditions
\begin{enumerate}
\item $E$ can be formulated as $E = A + A^T$ where $A=UV^T$ and $U, V \in \Real^{3n \times 3}$
\begin{align}  \label{eq:uv}
V =\left[\begin{array}{ccc}
R_{1}^T\\
\vdots\\
R_{n}^T
\end{array}\right] ~~ & ~~
U =\left[\begin{array}{c}
R_{1}^T T_{1}\\
\vdots\\
R_{n}^T T_{n}
\end{array}\right]
\end{align}
with $T_i = [{\bf t}_i]_{\times}$ and $\sum_{i=1}^n {\bf t}_i = 0$.
\item Each column of $U$ is orthogonal to each column of $V$, i.e., $V^T U=0_{3 \times 3}$

\item rank(V)=3 

\item If not all $\{{\bf t}_i\}_{i=1}^n$ are collinear, then rank(U) and rank(A) = 3. Moreover, if the (thin) SVD of $A$ is  $A=\hat{U}\Sigma\hat{V^{T}}$, with  $\hat {U}, \hat {V} \in \Real^{3n \times 3}$ and $\Sigma \in \Real^{3 \times 3}$ then
 the (thin) SVD of $E$ is $$E = \left[\hat{U},\hat{V}\right]\left(\begin{array}{cc}
\Sigma\\
 & \Sigma
\end{array}\right)\left[\begin{array}{c}
\hat{V}^{T}\\
\hat{U}^{T}
\end{array}\right]$$
implying  rank(E) = 6.
\end{enumerate}
\end{theorem}

\subsection{Main theoretical results}
In this section   we derive    and prove necessary and sufficient conditions for the consistency of $n$-view essential matrices in terms of their spectral decomposition. These conditions, in turn, will be  used  later to formulate a constrained optimization problem and to extract the motion parameters from a consistent $n$-view essential matrix $E$ . 
\begin{theorem}\label{consistencyThm} Let $E \in \mathbb{S}^{3n}$ be a consistent $n$-view fundamental matrix with a set of $n$ cameras  whose centers are not all collinear. We denote by $\Sigma_+ ,\Sigma_- \in \mathbb{R}^{3\times3}$ the diagonal matrices with   the 3 positive and 3 negative eigenvalues of $E$, respectively. Then, the following conditions are equivalent:
\begin{enumerate}
\item $E$ is a consistent $n$-view essential matrix
\item The (thin) SVD of $E$   can be written in the form  $$E=\left[ {\hat{U}},{\hat{V}}\right]\left(\begin{array}{cc}
\Sigma_{+}\\
 & \Sigma_{+}
\end{array}\right)\left[\begin{array}{c}
 {\hat{V}}^{T}\\
 {\hat{U}}^{T}
\end{array}\right]$$
with $\hat{U}, \hat{V} \in \Real^{3n \times 3}$ such that each $3 \times 3$ block of  $\hat{V}$, $\hat{V}_i$, $i=1, ..., n$, is an $\sqrt{n}$-scaled  rotation matrix, i.e., ${\hat V_i}=\frac{1}{\sqrt{n}} {\hat R}_i $, where ${\hat R}_i \in SO(3) $.  
We say that $\hat{V}$ is a block rotation matrix.

\item $\Sigma_+ = - \Sigma_-$ and the  (thin) spectral decomposition of $E$ is of the form 
\[ E=[X, Y]\left(\begin{array}{cc}
\Sigma_{+}\\
 & \Sigma_{-}
\end{array}\right)\left[\begin{array}{c}
 X^{T}\\
 Y^{T}
\end{array}\right]\]
such that $\sqrt{0.5}(X+Y)$ is a block rotation matrix.
\end{enumerate}
\end{theorem}

\begin{proof}

\textit{(1)$\Rightarrow$(2)}  Assume that $E$ is a consistent $n$-view essential matrix. Then, according to Thm.~\ref{thm:necessary_essential}, $E=A+A^T$ with $A=UV^T$ and $U, V \in \Real^{3n \times 3}$ which take the formas in \eqref{eq:uv}.
Since $A=UV^T$ and $rank(A)=3$, then $A^TA=VU^TUV$ and $A^TA \succeq 0$ with $rank(A^TA)=3$ ($A$ and $A^TA$ share the same null space). First, we construct  a spectral decomposition to $A^TA$, relying on the special properties of $U$ and $V$. We have $rank(U)=3$, and therefore $U^TU$, which is a  $3 \times 3$, symmetric positive semi-definite matrix, is of full rank. Its spectral decomposition is of the form $U^TU = Q D Q^T$, where $Q \in SO(3)$. (Spectral decomposition guarantees that $Q \in O(3)$. However, $Q$ can be replaced by $-Q$ if $\det(Q)=-1$.) $D \in \Real^{3 \times 3}$ is a diagonal matrix consisting of the (positive) eigenvalues  of $U^TU$. This spectral decomposition yields the following decomposition 
\begin{equation}\label{eq:decomposition}
A^T A = VQDQ^TV^T.
\end{equation}
Now, note that
\begin{align*}
Q^TV^T VQ & = Q^T \begin{bmatrix}R_1, ..., R_n \end{bmatrix} \begin{bmatrix}R_1^T \\ \vdots \\ R_n^T \end{bmatrix} Q 
=n I_{3 \times 3}. 
\end{align*}
By a simple manipulation \eqref{eq:decomposition} becomes a spectral decomposition  \begin{equation}\label{eq:spectral1}
 A^TA= \left(\frac{1}{\sqrt{n}}V\right) Q(nD)Q^T \left(\frac{1}{\sqrt{n}}V^T\right).\end{equation}
On the other hand, the (thin) SVD of $A$ is of the form $A ={\hat U} \Sigma {\hat V}^T$, where ${\hat U}, {\hat V} \in \Real^{3n \times 3}$, $\Sigma \in \Real^{3 \times 3}$. This means that 
\begin{equation}\label{eq:spectral2}
A^TA = {\hat V} \Sigma^2 {\hat V}^T.
\end{equation} 
Due to the uniqueness of the eigenvector decomposition, \eqref{eq:spectral1} and \eqref{eq:spectral2} collapse to the same eigenvector decomposition, up to some global rotation,  $H \in SO(3)$, that is  $\frac{1}{\sqrt n} VQ= {\hat V} H$, which means that ${\hat V}_i = \frac{1}{\sqrt{n}} R_i^TQ H^T$. Since $R_i^T, Q, H^{T} \in SO(3)$, then ${\hat R}_i=:R_i^T Q H^{T} \in SO(3)$, showing that ${\hat V}$ is a block rotation matrix. Finally, by Thm.~\ref{thm:necessary_essential}, the (thin) SVD of $E$ is of the form 

\begin{equation}\label{eq:E_SVD}
E = \left[\hat{U},\hat{V}\right]\left(\begin{array}{cc}
\Sigma\\
 & \Sigma
\end{array}\right)\left[\begin{array}{c}
\hat{V}^{T}\\
\hat{U}^{T}
\end{array}\right]
\end{equation}
and according to Lemma \ref{SVDtoSPEC}, 
the eigenvalues of $E$ are $\Sigma$ and $-\Sigma$. Since the elements on the diagonal of $\Sigma$ are positive, and $E$ is symmetric with exactly 3 positive eigenvalues $\Sigma_+$ and 3 negative eigenvalues $\Sigma_-$,  it follows that $\Sigma = \Sigma_+$ and $-\Sigma = \Sigma_-$ concluding the proof.

\textit{(2)$\Rightarrow$(1)} Let $E$ be a consistent $n$-view fundamental matrix that satisfies condition (2). We would like to show that $E$  is a consistent $n$-view essential matrix. By condition (2) $E$ can be written as
\begin{equation}\label{eq:E_hat}
E={\hat{U}}\Sigma_+{\hat{V}^{T}}+\hat{V}\Sigma_+{\hat{U}^{T}}=\bar{{U}}\hat{V}^T+\hat{V}\bar{U^{T}},
\end{equation}
where ${\bar U}={\hat U} \Sigma_+$ with  ${\hat V}_i = \frac{1}{\sqrt n} {\hat R}_i$, ${\hat R}_i \in SO(3)$. By definition  $E_{ii}=0$, and this implies that $\bar{U}_i\hat{V}_i^T$ is skew symmetric. Using  Lemma \ref{lemma:skew},  $\bar{U}_i=\hat{V}_i {\hat T}_i $ for some skew symmetric matrix ${\hat T}_i = [{\hat {\bf t}}_i]_{\times}$. Plugging ${\bar U}_i$ and ${\hat V}_i$ in \eqref{eq:E_hat} yields 
\begin{align*}
E_{ij}&=\bar{U}_i\hat{V}_j^T+\hat{V}_i\bar{U}_j^T \\
      &= \frac{1}{n}{\hat R}_i {\hat T}_i {\hat R}_j^T - \frac{1}{n}{\hat R}_i {\hat T}_j {\hat R}_j^T \\
&= {R_i}^T([{\bf t}_i]_{\times} - [{\bf t}_j]_{\times})R_j, 
\end{align*}
where $R_i = {\hat R}_i^T$ and  ${\bf t}_i = \frac{1}{n} {\hat{\bf t}}_i$, concluding the proof.

\textit{(2)$\Rightarrow$(3)} Let $E$ be an $n$-view fundamental matrix which satisfies condition (2). This means that the (thin) SVD\ of $E$ can be written in the form $E=\left[ {\hat{U}},{\hat{V}}\right]\left(\begin{array}{cc}
\Sigma_{+}\\
 & \Sigma_{+}
\end{array}\right)\left[\begin{array}{c}
 {\hat{V}}^{T}\\
 {\hat{U}}^{T}
\end{array}\right]$, where $\hat{V}$ is a block rotation matrix. Then, by Lemma \ref{SVDtoSPEC}, the (thin) spectral decomposition of $E$ is $E=[X, Y]\left(\begin{array}{cc}
\Sigma_{+}\\
 & -\Sigma_{+}
\end{array}\right)\left[\begin{array}{c}
 X^{T}\\
 Y^{T}
\end{array}\right]$, where $X, Y$ are the eigenvectors of $E$ satisfying \ $X = \sqrt{0.5}({\hat U}+{\hat V}) $ and $Y=\sqrt{0.5}({\hat V}-{\hat U})$. Since ${\hat V} = \sqrt{0.5} (X+Y)$, and by condition (2) ${\hat V}$  is a block rotation matrix, the claim is confirmed and also $\Sigma_- = - \Sigma_+$.

\textit{(3)$\Rightarrow$(2)} Let $E$ be a consistent $n$-view fundamental matrix satisfying condition (3), i.e.,  its  (thin) spectral decomposition is of the form $E=\left[ {X},{Y}\right]\left(\begin{array}{cc}
\Sigma_{+}\\
 & \Sigma_{-}
\end{array}\right)\left[\begin{array}{c}
 {X}^{T}\\
 {Y}^{T}
\end{array}\right]$, where $\sqrt{0.5}(X+Y)$ is a block rotation matrix. Since $\Sigma_+ = -\Sigma_-$, we can use  Lemma \ref{SVDtoSPEC}, which implies   that the (thin) SVD is of the form $E=\left[ {\hat{U}},{\hat{V}}\right]\left(\begin{array}{cc}
\Sigma_{+}\\
 & \Sigma_{+}
\end{array}\right)\left[\begin{array}{c}
 {\hat{V}}^{T}\\
 {\hat{U}}^{T}
\end{array}\right]$,
where ${\hat V} = \sqrt{0.5}(X+Y)$, concluding the proof.
\end{proof}

\begin{corollary} \label{EucRecon}
Euclidean reconstruction. Let $E$ be  a
consistent $n$-view essential matrix with 6 distinct eigenvalues, then it is  possible to explicitly determine $R_1 \dots R_n$ and ${\bf t}_1,\ldots,{\bf t}_n$ that are consistent with respect to $E.$

\end{corollary}
\begin{proof}
The claim is justified by the following construction,
which relies on the spectral characterizations   of Thm.  \ref{consistencyThm}.
\begin{enumerate}
\item Calculate the eigenvectors $X, Y$ of $E$, and  the corresponding three positive eigenvalues, $\Sigma_+$,  and three negative eigenvalues, $\Sigma_-$, respectively. 
\item  To realize condition (3) of Thm.~\ref{consistencyThm}, there are 8 possible choices of $\sqrt{0.5}(X+ YI_{s})$, where $I_s=\left(\begin{array}{ccc}
\pm1 & 0 & 0 \\
0 & \pm1 & 0 \\
0 & 0 & \pm1 \\ 
\end{array}\right)$, due to the sign ambiguity of each eigenvector. Then, $I_s$ is chosen such that $\sqrt{0.5}(X+ YI_{s})$ is  block rotation matrix up to a global sign which can be removed.  \item  This spectral decomposition determines  directly an SVD decomposition in the form of condition (2)  of Thm. \ref{consistencyThm}.  
We would like to emphasize that due to the multiplicity of singular values,   a standard SVD method which is performed directly on $E$, in general,  will not produce this special SVD pattern. 
 
\item Having the relation $E_{ij}=\hat{U}_i\Sigma_+ \hat{V}_j^T+\hat{V}_i\Sigma_+ \hat{U}_j^T$ and since $E_{ii}=0$ we get that $\hat{U}_i\Sigma_{+} \hat{V}_i^T$ is skew symmetric. We denote ${\bar U}_i={\hat U}_i \Sigma_+$ and, by Lemma  \ref{lemma:skew}, it holds that $ {\hat T}_i={\hat V}_i^{-1}{\bar U}_i$. 
\item Finally, for $i=1,2,..,n$,  define $R_i =: \sqrt{n}{\hat V}_i^T$ and  ${\bf t}_i =: \frac{1}{n} {\hat{\bf t}}_i$  and it holds that
\begin{align*}
E_{ij} & = {R_i}^T([{\bf t}_i]_{\times} - [{\bf t}_j]_{\times})R_j.
\end{align*}
 \end{enumerate}
 \end{proof}
This construction  yields $\{R_i\}_{i=1}^{i=n}$ and $\{{\bf t}_i\}_{i=1}^{i=n}$ which are consistent with respect to $E$. Moreover,  the reconstruction is unique up to a global similarity transformation. Roughly speaking, this can be proven for $n=3$ by applying an argument from \cite{holt1995uniqueness}. Next, for $n>3$,  by induction, suppose we obtain two reconstructions, $P_1,...,P_n$ and $P'_1,...,P'_n$. By the induction assumption these must include two sets of $n-1$ non-collinear cameras so that each is unique up to a similarity transformation. Such two sets overlap in at least 2 cameras, which in turn imply that the two similarity transformations must be identical. The complete proof is provided in the supplementary material.

\subsection{Supporting lemmas}

\begin{lemma}\label{lemma:skew}\cite{kasten2018gpsfm} Let $A,B\in \Real^{3\times3}$ with $rank(A)=2$, $rank(B)=3$ and $AB^{T}$ is skew symmetric, then
$T=B^{-1}A$ is skew symmetric.
\end{lemma}

\begin{lemma}\label{SVDtoSPEC} Let $E\in \mathbb{S}^{3n}$ of rank(6), and $\Sigma \in \Real^{3 \times 3} $, a diagonal matrix, with positive elements on the diagonal. Let $X, Y, U, V \in \Real^{3n \times 3}$, and we define the mapping $(X,Y) \leftrightarrow (U,V):$ $X = \sqrt{0.5}({\hat U}+{\hat V})$, $Y = \sqrt{0.5}({\hat V}-{\hat U})$, ${\hat U} = \sqrt{0.5}(X-Y), {\hat V} = \sqrt{0.5}(X+Y)$. 

Then, the (thin) SVD of $E$ is of the form $$E=\left[ {\hat{U}},{\hat{V}}\right]\left(\begin{array}{cc}
\Sigma\\
 & \Sigma
\end{array}\right)\left[\begin{array}{c}
 {\hat{V}}^{T}\\
 {\hat{U}}^{T}
\end{array}\right]$$   if and only if the (thin) spectral decomposition of $E$ is 
of the form 

$$E=[X, Y]\left(\begin{array}{cc}
\Sigma\\
 & -\Sigma
\end{array}\right)\left[\begin{array}{c}
 X^{T}\\
 Y^{T}
\end{array}\right]$$ 
\end{lemma}
\begin{proof}
The proof is provided in the supplementary material.
\end{proof}

\section{Method}
Given images $I_1,... ,I_n$, we assume a standard robust method is used to estimate the pairwise essential matrices, which we denote by $\Omega=\{\hat{E}_{ij}\}$. In practice, only a small subset of the pairwise essential matrices are estimated, due to occlusion, large viewpoint and brightness changes  as well as objects' motion, and in addition the available estimates are noisy.  
Our goal therefore is to find a consistent $n$-view essential matrix $E\in \mathbb{S}^{3n}$ that is as similar to the measured essential matrices as possible. 

To make an $n$-view essential matrix consistent, its blocks of pairwise essential matrices must each be scaled by an unknown factor. 
\cite{sengupta2017new} proposed an optimization scheme that explicitly seeks for the unknown scale factors, yielding a nonlinear, rank-constrained optimization  formulation.  
The success of this  approach critically  depends on the quality of its initialization, which in experiments was obtained by applying another, state of the art  SfM method.

 More recently, \cite{kasten2018gpsfm} proposed an analogous approach for projective SfM. They showed that  a consistent  $3$-view fundamental matrix, which uniquely determines camera matrices (up to a projective ambiguity) from a triplet of images,  is invariant to  scaling  of its pairwise fundamental matrices. This  allowed them to formulate an optimization problem that seeks $3$-view  fundamental matrices that are both consistent and compatible, while avoiding the need to explicitly  optimize for the scale factors. 

 In this paper, we introduce an optimization scheme that is analogous to that of \cite{kasten2018gpsfm}, but adapted to calibrated settings. In particular, our scheme uses the algebraic constraints derived in Thm.~\ref{consistencyThm} to enforce the consistency of noisy, and possibly partial essential matrices. Similar to \cite{kasten2018gpsfm}, our method simultaneously enforces consistency of camera triplets attached rigidly to each other, allowing us to avoid optimizing explicitly for the unknown scales of the estimated essential matrices. (To that end we further extend Thm.~\ref{consistencyThm} to handle scaled rotations for image triplets, see supplementary material for details.) Our formulation, however, is more involved than in \cite{kasten2018gpsfm} due to the additional spectral constraints required for Euclidean reconstruction.


In the rest of this section we present our constrained optimization formulation and propose   an
ADMM-based solution scheme. Subsequently, we discuss how to select minimal subsets of
triplets to speed up the optimization. Finally, we show how
the results of our optimization can be used to reconstruct the absolute orientations and positions of the
$n$ cameras.         

\subsection{Optimization}\label{sec:optimization}
In multi-view settings, it is common to define a viewing graph $G=(V,W)$, with nodes $v_1, \ldots, v_n$, corresponding to the $n$ cameras, and $w_{ij} \in W$ if $ {\hat E}_{ij} $ is one of the estimated pairwise essential matrices.
Let $\tau$ denote a collection of $m$ 3-cliques of cameras where $m\leq (^n_3).   $ The collection may be the full set of the 3-cliques in $G$, or a chosen subset as  described in Sec.\ \ref{sec:GraphConstruction}. We index the elements of $\tau$ by $k=1,...,m$, where $\tau(k)$ denote the $k^{th}$ triplet. The collection $\tau$ determines  a partial selection of measured essential matrices, $\Omega$, that  plays a role in the optimization problem, where it holds that if $\hat{E}_{ij}\in \Omega $ then $\hat{E}_{ij}^T = \hat{E}_{ji}\in \Omega. $

We define the measurements matrix $\hat{E}\in \mathbb{S}^{3n}$ to have $\hat{E}_{ij}$ as its $(i,j)^{th}$ block if $\hat{E}_{ij}\in \Omega$ and $0_{3\times 3} $ in  the rest of its blocks. In our optimization problem we look for $E$ that is as close as possible to $\hat{E }$ under the constraints that its $9\times 9 $ blocks, induced by $\{\tau (k)\}_{k=1}^m$ and denoted by  $\{E_{\tau (k)}\}_{k=1}^m$, are consistent 3-view essential matrices. In general, such $E$ is inconsistent and incomplete, but  as we explain in Sec. \ref{subsec:retrieval}  it is possible to retrieve a set of $n$ absolute rotations and translations that is compatible with its essential matrices up to scale, which in turn implies that  the completion of the missing entries is  consistent. 

We formulate our constrained optimization as follows  
\begin{align} \label{eq:oprimization_total_objective}
& \underset{E\in{\cal E}}{\text{minimize}}
& & \sum_{ k= 1}^m ||E_{\tau(k)}-\hat E_{\tau(k)}||^2_F   \\
& \text{subject to} & &  rank(E_{\tau(k)})=6   \nonumber\\  
& & & \Sigma _+(E_{\tau(k)})=-\Sigma _-(E_{\tau(k)})  \nonumber \\& & & X(E_{\tau(k)})+Y(E_{\tau(k)})\text{ is block rotation,}    \nonumber
\end{align}
with $i=1, ..., n$ and $k=1, ..., m$, where $\Sigma_+(E_{\tau(k)}), \Sigma _-(E_{\tau(k)}) \in \Real^{3}$ denote the 3 largest (descending order) and $3$ smallest (ascending order) eigenvalues  of $E{_{\tau(k)}}$ respectively and  $X(E_{\tau(k)}) \in \Real^{9 \times 3}$ and $Y(E_{\tau(k)}) \in \Real^{9 \times 3}$ are their corresponding eigenvectors.

Solving \eqref{eq:oprimization_total_objective} is challenging due to its rank  and spectral  decomposition  constraints. We solve  this optimization problem using ADMM. To that end, as part of the  ADMM method \cite{boyd2011distributed}  $4m$ auxiliary matrix variables of size $9\times9$ are added: $2m$  variables duplicating $\{E_{\tau (k)}\}_{k=1}^{m}$, denoted $B = \{B_k\}_{k=1}^{m}$ and $D = \{D_k\}_{k=1}^{m}, $ as well as $2m$   Lagrange multipliers, $\Gamma = \{\Gamma_k\}_{k=1}^{m}$ and $\Phi = \{\Phi_k\}_{k=1}^{m}$. 
This yields the following constrained optimization problem
\begin{align} \label{eq:oprimization_total_objective_ADMM}
& \underset{\Gamma,\Phi}{\text{max}}~~ \underset{E,B,D}{\text{min}}
& & \sum_{ k= 1}^m L(E_{\tau{(k)}},B_k,\Gamma_k,D_k,\Phi_k)   \\
& \text{subject to}& & E\in{\cal E} & \nonumber \\ 
& & &  rank(B_{k})=rank(D_{k})=6   \nonumber\\  
& & & \Sigma _+(B_k)=-\Sigma_-(B_k)   \nonumber
\\& & & X(D_k)+Y(D_k)~\text{is block rotation}\nonumber
\end{align}
with $i=1, ..., n$ and $k=1, ..., m$, 
where 
\begin{align*}
L(E_{\tau{(k)}},B_k,\Gamma_k,D_k,\Phi_k)= \norm{E_{\tau{(k)}}-\hat E_{\tau{(k)}}}^2_F+ \\ \frac {\alpha_1}{2} \norm{B_k-E_{\tau{(k)}}+\Gamma_k}_F^2+ \frac{\alpha_2}{2} \norm{D_k-E_{\tau{(k)}}+\Phi_k}^2_F. \nonumber
\end{align*}
We initialize the auxiliary variables   at $t=0$ with
 \begin{align}
 B_k^{(0)}={\hat E}_{\tau{(k)}},D_k^{(0)}={\hat E}_{\tau{(k)}},\Gamma_k^{(0)}=0,\Phi_k^{(0)}=0. \nonumber
 \end{align}
Then, we solve the optimization problem iteratively by alternating between the following four steps.

\noindent\textbf{(i) \uline{Solving for $E$}}.
\begin{align}
&E^{(t)}= \underset{E}{\text{argmin}}
& & \sum_{k=1 }^m \norm{E_{\tau{(k)}}-\hat E_{\tau{(k)}}}^2_F \\
& & & + \frac{\alpha_1}{2} \norm{B^{(t-1)}_k-E_{\tau{(k)}}+\Gamma^{(t-1)}_k}_F^2\nonumber\\
& & & + \frac{\alpha_2}{2}\norm{D^{(t-1)}_k-E_{\tau{(k)}}+\Phi^{(t-1)}_k}_F^2\nonumber\\
& \text{subject to}
& & E\in{\cal E}  \nonumber 
\end{align}

This step in the optimization is solved by first neglecting the constraint of two identical singular values for each block, yielding a convex quadratic problem with a closed form solution.
Then, based on SVD, each $3\times 3$ block is modified to satisfy the constraint of two identical singular values.

\noindent\textbf{(ii) \uline{Solving for $B_{k }$}}. For all $k=1,... ,m$ \begin{align} \label{eq:oprimization_total_objective_ADMM}
& B_k^{(t)}= \underset{B_{k}}{\text{argmin}}
& & ||B_k-E_{\tau{(k)}}^{(t)}+\Gamma_k^{(t-1)}||_F^2   \\
& \text{subject to}
& &  rank(B_{k})=6 \nonumber\\  
& & & \Sigma _+(B_k)=-\Sigma _- (B_k)\nonumber   
\end{align}

The minimizer for this sub-problem is obtained in the following way. By construction, $E_{\tau{(k)}}^{(t)}-\Gamma_k^{(t-1)}$ is a symmetric matrix, and we denote its (full) spectral decomposition by  $U\Sigma U^T$, where $U \in \Real^{9 \times 9}$ and $\Sigma \in \Real^{9 \times 9}$ is a diagonal matrix in which the eigenvalues are arranged in a descending order.  
Then, the update is 
\begin{align}
B_k^{(t)}=U\Sigma ^{*}U^T,
\end{align}
where  $\Sigma ^{*} \in \Real^{9 \times 9}$ is a diagonal matrix with the entries
\begin{align}
\Sigma_{ii}^*=\begin{cases}\frac{1}{2}(\Sigma_{ii}-\Sigma_{10-ii}) & i\neq4,5,6 \\0 & i=4,5,6\end{cases}
\end{align}

\noindent{\textbf{(iii) \uline{Solving for $D_{k }$}}. For all $k=1,... ,m$ \begin{align} \label{eq:oprimization_total_objective_D}
& \underset{}{\text{}} D_k^{(t)}= \underset{D_{k}}{\text{argmin}}
& & ||D_k-E_{\tau{(k)}}^{(t)}+ \Phi_k^{(t-1)}||_F^2   \\
& \text{subject to}
& &  rank(D_{k})=6 \nonumber\\
& & & X({D_{k})}+Y({D_k})~\text {is a block rotation matrix} \nonumber \end{align}
We minimize this sub-problem by an iterative process, which we repeat until convergence.  We begin with $D_k = E_{\tau(k)}^{(t)} - \Phi_k^{(t-1)}$, which is symmetric by construction.  We apply spectral decomposition to $D_k$, and extract $X(D_k)$, $Y(D_k)$, $\Sigma_+(D_k)$ and $\Sigma_-(D_k)$. Assuming no eigenvalue multiplicities, the eigenvectors are determined uniquely up to a sign (this argument is justified in \ref{sec:experiments}). We denote by $I_s$,  a diagonal matrix of size $3 \times 3$, such that each diagonal element is either 1 or -1. There are eight configurations for $I_s$ from which we select the best, in the sense that on average  each $3 \times 3$ block, of the form, $\sqrt{0.5}[X+YI_s]_i$, $i=1,2,3$, is close to scaled rotation, using the following score,   
\begin{align*}
  &  I_s^*= \underset{I_s}{\mathrm{argmax}}
  \sum_{i=1}^3 \frac{ \norm{diag((X_i+Y_i I_s)^{T}(X_i+Y_i I_s))}_2}
  {\|(X_{i}+Y_{i}I_s)^T(X_{i}+Y_{i}I_s)\|_F}. \end{align*}
Next, let $V=\begin{bmatrix}V_1 \\V_2 \\V_3\end{bmatrix}$ be the projection of $\sqrt{0.5}(X+YI_s^{*})$ so that $V_i$ is the closest scaled rotation to $\sqrt{0.5}[X+YI_s^{*}]_i$. Projection to scaled $SO(3)$ is obtained through removal of the singular values from the SVD decomposition, setting the scale factor to the average of the singular values, and possibly negating the scale factor to make the determinant positive. Let $U=\sqrt{0.5}(X-YI_s^{*})$, and $\tilde{X} = \sqrt{0.5}(U+V)$ and ${\tilde Y}=\sqrt{0.5}(V-U)$. We then update the value of $D_k$ to be the symmetric matrix 
$$D_k=[{\tilde X}, {\tilde Y}]\left(\begin{array}{cc}
\Sigma_{+}\\
 & \Sigma_-
\end{array}\right)\left[\begin{array}{c}
 {\tilde X}^{T}\\
 {\tilde Y}^{T}
\end{array}\right]$$   
and repeat these steps until convergence. 

\noindent \textbf{(iv) \uline{Updating $\Gamma_{k },\Phi_{k}$}}. For all $k=1,\dots ,m$
\begin{align} 
\Gamma_k^{(t)}=\Gamma_k^{(t-1)}+ B_k^{(t)}-E_{\tau(k)}^{(t)}\\
\Phi_k^{(t)}=\Phi_k^{(t-1)}+D_k^{(t)}-E_{\tau(k)}^{(t)}
\end{align}

\subsection{Graph construction and outliers removal}\label{sec:GraphConstruction}

As explained above, to apply our optimization algorithm, it is required to determine a collection $\tau$ of camera triplets, which is a subset of the given camera triplets. The selection of a subset allows for better  efficiency and robustness. Similarly to \cite{klopschitz2010robust,kasten2018gpsfm}, we consider a weighted viewing graph $G$ whose weights for each edge $w_{ij}$ is assigned to be the number of  the inlier matches relating $I_i$ and $I_j$. We begin by selecting  disjoint maximal spanning trees from $G$, from which we extract an initial  subset of triplets.   We then, remove near collinear and inconsistent\ triplets. We next build a triplet graph $G_T$ whose nodes, which represent image triplets, are connected by an edge whenever two triplets share the same two cameras. Finally, we greedily remove nodes from $G_T$.  starting with the least consistent triplet (using the rotation consistency score defined below), a node is removed as long as the connectivity of $G_T$ is preserved and the total number of cameras associated  with $G_T$  does not decrease. 

To define collinear and consistency scores for each triplet we denote  the angles   in the triangle formed by three cameras  $i,j,k$  by $\theta_i,\theta_j,\theta_k$ respectively. We measure each angle using the known relative translations $\mathbf{t}_{ij},\mathbf{t}_{ik},\mathbf{t}_{jk}$,  i.e., $\cos\theta_i=(\frac{\mathbf{t}_{ij}^T \mathbf{t}_{ik}}{||\mathbf{t}_{ij}||||{\mathbf{t}_{ik}}||})$.   Then, the   collinearity score of cameras $\{i,j,k\}$ is the minimal angle in $\{\theta_i,\theta_j,\theta_k\}$. The consistency score of translations  is defined by $|\theta_i+ \theta_j+\theta_k-\pi|$ and   the consistency score of rotations by $||R_{ij}R_{jk}R_{ki}-I||$.

\subsection{Location and orientation retrieval}\label{subsec:retrieval}
 After solving \eqref{eq:oprimization_total_objective}, we extract from $E$ the collection of 3-view essential matrices $\{E_{\tau(k)}\}_{k=1}^{m}$, which, due to the optimization, are consistent w.r.t scaled rotations. Next, using  Corollary \ref{EucRecon}  with additional   block normalizing at step 3, three rotations $\{ R^{\tau(k)}_1,R^{\tau(k)}_2,R^{\tau(k)}_3\}$ and three translations $\{t^{\tau(k)}_1,t^{\tau(k)}_2,t^{\tau(k)}_3\}$ are extracted from each $E_{\tau(k)}$, which are  uniquely defined up to a similarity transformation. Now, any two triplets in $\tau$ that share two cameras $a,b$  agree on $E_{ab}$. Since the sign of  $E_{ab}$ is fixed, it determines the cameras $a,b$ up to 2 unique configurations  \cite{hartley2003multiple}. Therefore, each one of the two triplets must agree with one of the two configurations. As a result, assuming that both triplets defines the same configuration for $a$ and $b$, there is a unique similarity transformation between the two triplets. In practice, in our experiments we observe that this is always the case.

By the construction process described in Sec.~\ref{sec:GraphConstruction}, the collection of triplets  $\tau$ form a connected triplet graph. It is therefore  possible to traverse the graph and  apply a similarity transformation  on the three cameras of each new node  $\tau(k),$ and as a result bring  all the cameras to a common Euclidean frame.

\section{Experiments}

To evaluate our approach, we implemented the SfM pipeline described next and  tested it on ten challenging collections of  unordered internet photographs of various sizes from \cite{wilson2014robust}. Each dataset is provided with ground truth camera parameters. 
We use our method to recover    camera parameters and compare them to the parameters obtained with existing methods, before and after bundle adjustment (BA).
\begin{table}[t] \tiny

\resizebox{\linewidth}{!}{%
\

\begin{tabular}{|l|c|c|c|c|c|c|}

\hline
 &\multicolumn{2}{|c|}{ \textbf{ Our Method } } &    \multicolumn{2}{|c|}{ \textbf{Chatterjee et al.} \cite{chatterjee2018robust}} &\multicolumn{2}{|c|}{ \textbf{ Martinec et al. }\cite{martinec2007robust} }   \\\hline
\textbf{Data Set} &\textbf{ $R_f$} &\textbf{ $R_d$} & \textbf{$R_f$} &\textbf{ $R_d$} & \textbf{$R_f$} &\textbf{ $R_d$} \\\hline
Vienna Cathedral & \textbf{0.1141} & \textbf{4.7328} & 0.1514 & 7.8472 & - & - \\\hline
Piazza del Popolo & \textbf{0.0595} &  \textbf{2.4098}  & 0.2287 & 11.6022 & 0.5901 &  27.2359 \\\hline
NYC Library & \textbf{0.1200} & \textbf{4.8751}  & 0.1226 & 6.0765 & 0.5395 &  28.4385 \\\hline
Alamo & \textbf{0.0751} & \textbf{3.0489}  & 0.0879 & 4.4958 &  0.1503 & 7.3180 \\\hline
Metropolis & \textbf{0.30} & \textbf{15.6613}  & 0.4612 & 26.5636 & 1.1381  &  58.6143 \\\hline
Yorkminster & \textbf{0.1499} &  \textbf{6.6343} & 0.1526& 7.8017 & 1.3434 & 73.3679\\\hline
Montreal ND & \textbf{0.0608} &  \textbf{2.4652}  & 0.1049 & 5.8742 & 0.2419  & 11.928\\\hline
Tower of London & \textbf{0.1250} & \textbf{5.0731} & 0.1366 & 5.8872 &  0.3435 & 16.6457 \\\hline
Ellis Island & 0.0636 & 2.5784 & \textbf{0.0499} & \textbf{2.2486} & 0.0616 & 2.4955  \\\hline
Notre Dame & \textbf{0.0619}  & \textbf{2.5091} & 0.0876 & 4.7463 & 0.2109 & 10.5894 \\\hline
\end{tabular}
}
\caption{\small Errors in estimated camera orientations, compared to ground truth measurements. $R_f$ denotes the mean Frobenius norm error,  averaged over the different cameras, and $R_d$ is the mean angular error in degrees. Empty cells represent missing information.}\label{table:rotations}

\end{table}

\begin{table*}[tbh]
\tiny

\resizebox{\linewidth}{!}{%

\begin{tabular}{|l|c|c|c|c|c|c|c|c|c|c|c|c|c|c|c|c|c|c|c|c|c|}\hline
 &&\multicolumn{5}{|c|}{ \textbf{ Our Method } } &    \multicolumn{5}{|c|}{ \textbf{LUD} \cite{ozyesil2015robust}} &\multicolumn{4}{|c|}{ \textbf{ 1DSFM }\cite{wilson2014robust} } &\multicolumn{4}{|c|}{ \textbf{Cui }\cite{cui2015linear} }  \\\hline
\textbf{Data set} & $N_c$ & $\bar{x}$ & $\tilde{x}$ &$\bar{x}_{BA}$ & $\tilde{x}_{BA}$ & $N_r$ & $\bar{x}$ & $\tilde{x}$ &$\bar{x}_{BA}$ & $\tilde{x}_{BA}$ & $N_r$  & $\tilde{x}$ &$\bar{x}_{BA}$ & $\tilde{x}_{BA}$ & $N_r$  & $\tilde{x}$ &$\bar{x}_{BA}$ & $\tilde{x}_{BA}$ & $N_r$ \\\hline
Vienna Cathedral & 836 & \textbf{9.6} & 4.2 & 5.4 & 1.2  & 674 & 10 & 5.4 & 10 & 4.4 & 750  & 6.6 & 2e4 & \textbf{0.5} & 757  & \textbf{3.5} & \textbf{4.0} & 2.6 & 578\\\hline
Piazza del Popolo & 328 & 7.2 &  3.5 & \textbf{2.5} & \textbf{0.8} & 275 & \textbf{5} & \textbf{1.5} & 4 & 1.0 & 305  & 3.1 & 200 &2.6  & 303  & 2.6 & 3.2 & 2.4 & 294 \\\hline
NYC Library & 332 & \textbf{3.3} &  2.2 & \textbf{1.1} & 0.47 & 277 & 6 & \textbf{2.0} & 7 & 1.4 & 320  & 2.5 & 20 & \textbf{0.4} & 292  & 1.4  & 6.9 & 0.9 
& 288 \\\hline
Alamo & 577 & 2.5  & 1.2  & \textbf{0.8} & 0.35 & 482 & \textbf{2} & \textbf{0.4} & 2 & \textbf{0.3} & 547 & 1.1 & 2e7 & \textbf{0.3} & 521  & 0.6 & 3.7 & 0.5 & 500 \\\hline
Metropolis & 341 & 15.2 & 6.9 & \textbf{2.7} & 1.4 & 168 & \textbf{4} & \textbf{1.6} & 4 & 1.5 & 288  & 9.9 &  70& \textbf{1.2} & 288  & - & - & - & -\\\hline
Yorkminster & 437 & 5.6 & \textbf{2.7}  & \textbf{1.9} & 0.8 & 341 & \textbf{5} & \textbf{2.7} & 4 & 1.3 & 404  & 3.4 &  500&\textbf{0.2}  & 395  & 3.7 & 14 & 3.8 & 341 \\\hline
Montreal ND & 450 & 1.9 &  1.0 & \textbf{0.6} & \textbf{0.4} & 416 & \textbf{1} & \textbf{0.5} & 1 & \textbf{0.4} & 435  & 2.5 &  1 & 0.9 & 425  & 0.8 & 1.1 & \textbf{0.4} & 426 \\\hline
Tower of London & 572 & \textbf{11.6} & 5.0 & \textbf{4} & 1.0 & 414 & 20 & \textbf{4.7} & 10 & 3.3 & 425 & 11 &   40&\textbf{0.4}  & 414  & 4.4 & 6.2 & 1.1 & 393 \\\hline
Ellis Island & 227 & 14.1 & 6.1 & 5.3 & 1.7 & 211 & - & - & - & - & -  & 3.7 & 40 & \textbf{0.4} & 213 & 3.1 & \textbf{1.8} & 0.6 & 211 \\\hline
Notre Dame & 553 & 1.8 & 0.8 & \textbf{0.4} & \textbf{0.2} & 529 & \textbf{0.8} & \textbf{0.3} & 0.7 & \textbf{0.2} & 536  & 10 & 7&  2.1 & 500  & \textbf{0.3} & 0.8 & \textbf{0.2} & 539 \\\hline
\end{tabular}
}
\caption{\small Camera positions error in meters evaluated on the data sets of \cite{wilson2014robust}. $N_c$ is the number of images in each dataset,  $\bar{x},\tilde{x}$ are the mean and median error respectively  before bundle adjustment, and $\bar{x}_{BA},\tilde{x}_{BA}$ are the mean and median errors after bundle adjustment. $N_r$ are the number of reconstructed cameras. Empty cells represent missing information.}
\label{table:translation}
\end{table*}

\begin{table*}[tbh] 
\tiny

\resizebox{\linewidth}{!}{%
\
\begin{tabular}{|l|c|c|c|c|c|c|c|c|c|c|c|c|}\hline
&\multicolumn{3}{|c|}{ \textbf{ Our Method } } &    \multicolumn{3}{|c|}{ \textbf{LUD} \cite{ozyesil2015robust}} &\multicolumn{3}{|c|}{ \textbf{ 1DSFM }\cite{wilson2014robust} } &\multicolumn{3}{|c|}{ \textbf{Cui }\cite{cui2015linear} }  \\\hline
\textbf{Data set} & $T_{R+T}$ & $T_{BA}$ & $T_{Tot}$ & $T_{R+T}$ & $T_{BA}$ & $T_{Tot}$ &$T_{R+T}$ & $T_{BA}$ & $T_{Tot}$ & $T_{R+T}$ & $T_{BA}$ & $T_{Tot}$  \\\hline
Vienna Cathedral & \textbf{68} & 293 & \textbf{566} & 787 & 208 & 1467 & 323 & 3611 & 3934 & - & 717 & 959 \\\hline
Piazza del Popolo & \textbf{26} & 27 & \textbf{87} & 88 & 31 & 162 & 42 & 213 &  255 & - & 93 & 144\\\hline
NYC Library & \textbf{28} & 58 & 125 & 102 & 47 & 200 & 47 & 382 & 429 & - & 48 & \textbf{90}\\\hline
Alamo & \textbf{47} & 155 & \textbf{327} & 385 & 133 & 750 & 152 &  646 &  798 & - & 362 & 621\\\hline
Metropolis & \textbf{16} & 70 & \textbf{93} & 67 & 38 & 142 & 47 & 224 &  271 & - & - & -\\\hline
Yorkminster & \textbf{33} & 116 & 207 & 103 & 148 & 297 & 71 & 955 & 1026 & -& 63 & \textbf{108} \\\hline
Montreal ND & \textbf{41} & 170 & 494 & 271 & 167 & 553 & 93 & 1043  &  1136 & - & 226 & \textbf{351}\\\hline
Tower of London & \textbf{41} & 120 & 241 & 88 & 86 & 228 & 61 & 750 & 811 & - & 121 & \textbf{221}\\\hline
Ellis Island & \textbf{21} & 53 & 140 & - & - & - & 29 & 276 &  305 & - & 64 & \textbf{95} \\\hline
Notre Dame & \textbf{52} & 277 & \textbf{720} & 707 & 126 & 1047 & 205 & 2139 &  2344 & - & 793 & 1159 \\\hline
\end{tabular}
}
\caption {\small Runtime in seconds for results in Table \ref{table:translation}. $T_{R+T}$ denotes the time for motion averaging (for other methods rotation and translation estimation). $T_{BA}$ the time for bundle adjustment and $T_{Tot}$ is the total running time of the method, including the additional time for building the triangle cover. Empty cells represent image collections not tested by the authors. In addition, Cui \cite{cui2015linear} does not report results before BA.}\label{table:Times}
\end{table*}

\subsection{SfM pipeline}
\label{sec:experiments}

 The input to our algorithm is a set of independently estimated pairwise essential matrices, along with the number of inlier matches for each pairwise essential matrix. We build a triplet graph $G_T$ as we describe in Sec.~\ref{sec:GraphConstruction}, removing any triplet whose (a) collinearity score is below 0.17 radians, (b)  rotation consistency score exceeds 1.1, or (c)  translation consistency score exceeds 1 radians.  The final connected graph $G_T$ defines the collection $\tau$ of triplets of cameras. 

Next we apply our optimization algorithm as is  described in Sec.~\ref{sec:optimization}. During optimization we observed that the eigenvalues of $B_k$ and $D_k$ \eqref{eq:oprimization_total_objective_ADMM} were always distinct for  $k=1,\ldots,m$. This means that  the optimization variables,  $E_{\tau(k)}$,  indeed converge to $3$-view consistent essential matrices with distinct eigenvalues. At this stage, we follow Corollary \ref{EucRecon}   to  recover camera positions  and orientations for each triplet of cameras and align all the recovered camera matrices by similarity transformations, as is described in Sec.~\ref{subsec:retrieval}.   This yields a set of  cameras with absolute positions and orientations. The obtained  camera parameters   are finally refined  using   BA.

\subsection{Results}

 To evaluate our recovered camera orientations, we compare our results to those obtained with the methods of Chatterjee et al. \cite{chatterjee2018robust}  and Martinec et al. \cite{martinec2007robust}. For a fair comparison we evaluate these methods on same subset of images used in our method. Moreover, since in contrast to our method  these solvers do not estimate camera positions we evaluate the results before BA.  The results are summarized in Table \ref{table:rotations}. Our method outperforms these two solvers in nine out of the ten datasets.

 To evaluate our recovered camera  positions we compare our method to the following  position solvers: Cui et al. \cite{cui2015linear},  1DSFM \cite{wilson2014robust} and LUD \cite{ozyesil2015robust}. The results are summarized in Table \ref{table:translation}. Note that \cite{ozyesil2015robust,cui2015linear}  use point matches in their pipelines, while both our method and  \cite{wilson2014robust}  do not use point correspondences until the final BA. In general, the latter approaches allow for faster optimization, but result in inaccuracies before BA. On the other hand, it allows for greater improvement in the final BA, compared to \cite{ozyesil2015robust,cui2015linear}. Indeed, as can be seen in the table, while our method surpasses all the other methods before bundle adjustment in three out of the ten datasets (according to mean error), it achieves superior performance on eight out of ten after BA. 

Table \ref{table:Times} further compares execution times, before and after BA, showing that our method is efficient, compared to existing methods.  

\subsection{Technical details}
For Bundle Adjustment we used the   Theia standard SfM library~\cite{theia-manual}. Camera position results for \cite{ozyesil2015robust,cui2015linear,wilson2014robust} in Tables    \ref{table:translation} and \ref{table:Times} are taken from the papers. We ran our experiments on an Intel(R)-i7 3.60GHz with Windows. Bundle Adjustment was performed  on a Linux machine Intel(R) Xeon(R) CPU @ 2.30GHz with 16 cores.

\section{Conclusion}

We have provided in this paper algebraic conditions for the consistency of essential matrices in multiview settings and an algorithm for their averaging given noisy and partial measurements. In future research we will seek to further incorporate collinear camera triplets in the averaging algorithm, explore numerical properties, and design online consistency enforcement algorithms for SLAM settings.

\vspace{2mm}

\noindent \textbf{Acknowledgment} This research was supported in part by the Minerva foundation with funding from the Federal German Ministry for Education and Research.

{\small
\bibliographystyle{ieee_fullname}
\bibliography{egbib}
}

\setcounter{section}{0}
\title{Algebraic Characterization of Essential Matrices and Their Averaging\\
in Multiview Settings\\~\\Supplementary Material}
\author{Yoni Kasten* \hspace{1cm} Amnon Geifman* \hspace{1cm} Meirav Galun  \hspace{1cm}Ronen Basri \\
Weizmann Institute of Science\\
{\tt\small  \{yoni.kasten,amnon.geifman,meirav.galun,ronen.basri\}@weizmann.ac.il}
}
\maketitle






\makeatletter
\def\blfootnote{\xdef\@thefnmark{}\@footnotetext}
\makeatother
\blfootnote{*Equal contributors}
\section{Counter example}
In the introduction of the paper we claim that a consistent $n$-view fundamental matrix whose $3 \times 3$ blocks form essential matrices does not necessarily form a consistent $n$-view essential matrix. We justify this argument by a constructing a counter example for the case $n=3$.

First note that the following observation is true. If ${\bf t}_2 ,{\bf t}_3, {\bf a}, {\bf b}\in \Real^3 $ and $R_2,R_3\in SO(3)$ and we set
$$V_1=I_{3 \times 3},V_2 =R_2^T,V_3 =R_3^T+ {\bf a} {\bf t}_3^T,{\bf t}_1=0_{3 \times 1}$$
and
$$F_{ij}=V_i ([{\bf t}_i-{\bf t}_j]_{\times})V_j^T,~~~~ i,j=1,2,3$$
then by construction
$$
F=
  \begin{bmatrix}
    0 & F_{} & F_{13} \\
    F_{12}^T & 0 & F_{23} \\
    F_{13}^T& F_{23}^T& 0 
  \end{bmatrix}
$$
is a consistent $n$-view fundamental matrix, unless the outer product ${\bf a} {\bf t}_3^T$ is such that it reduces the full rank of $V_3$ to 2. 

Examine each of the block matrices $F_{ij}$:
\begin{align*}
F_{12}&=V_1[{\bf t}_1-{\bf t}_2]_{\times} V_2^T=[-{\bf t}_2]_{\times} R_2\\
F_{13}&=V_1[{\bf t}_1-{\bf t}_3]_{\times} V_3^T=[-{\bf t}_3]_{\times}(R_3+ {\bf t}_3 {\bf a}^T)=[-{\bf t}_3]_{\times} R_3\\
F_{23}&=V_2[{\bf t}_2-{\bf t}_3]_{\times} V_3^T=R_2^T[{\bf t}_{2}-{\bf t}_{3}]_{\times} (R_3+{\bf t}_3 {\bf a}^T)
\end{align*}
It follows from this derivation that $F_{12}$ and $F_{13}$ are essential matrices. Next, we use the identity $[R {\bf t}]_\times=R[{\bf t}]_\times R^T$, which holds for any  $R \in SO(3)$ and ${\bf t} \in \Real^3$, to show that for any ${\bf b}\in \Real^3$   
\begin{align*}
F_{23} &=[R_2^T({\bf t}_2-{\bf t}_3)]_{\times} R_2^T(R_3+{\bf t}_3{\bf a}^T)\\ &= [R_2^T({\bf t}_2-{\bf t}_3)]_{\times}(R_2^T(R_3+{\bf t}_3 {\bf a}^T)+R_2^T({\bf t}_2-{\bf t}_3){\bf b}^{T})\\ &=[R_2^T({\bf t}_2-{\bf t}_3)]_{\times} R_2^T(R_3+ {\bf t}_3 {\bf a}^T+({\bf t}_2-{\bf t}_3){\bf b}^T). \end{align*}
The trivial choice of   ${\bf a} = {\bf b}=0$ yields an essential matrix $F_{23}$   which is consistent with the essential matrices $F_{12}$ and $F_{13}$. Our aim is to make a specific choice of  $R_2, R_3, {\bf t}_2, {\bf t}_3, {\bf a}, {\bf b}$ such that  $F_{23}$ will be an essential matrix that is inconsistent with the essential matrices $F_{13}$ and $F_{12}$.  
\subsection{Choosing $R_2, R_3, {\bf t}_2, {\bf t}_3, {\bf a}, {\bf b}$}
We first give some intuition for the way  we  set the values of $R_2, R_3, {\bf t}_2, {\bf t}_3, {\bf a}, {\bf b}$. We look at the term 
\begin{equation}\label{eq:termSM}
R_2^T(R_3+ {\bf t}_3 {\bf a}^T+({\bf t}_2-{\bf t}_3){\bf b}^T)
\end{equation}
and wish to set the values such that  
this term does not collapse to $R^T_2 R_3$, but is still in the form of $R_2^T R^*$, for some  $R^* \in SO(3)$.  To that end, we look for two rotation matrices $R^*$ and $R^{**}$  such that $M=R^{*}-R^{**}$ has  rank $2$, and set the values such that the term in \eqref{eq:termSM} will be equal to $R_2^T R^*$.  

Following the construction of $M$, the SVD of $M$ is of the form \begin{equation*}
M=U  \begin{bmatrix}
    \sigma_1 &0 & 0 \\
    0 & \sigma_2 & 0 \\
    0& 0& 0 
  \end{bmatrix} V^T
\end{equation*}    
and we set 
 \begin{equation}{\bf t}_3={\bf u}_1\sigma_1, {\bf a}= {\bf v}_1, {\bf b} = {\bf v}_2, {\bf t}_2 = {\bf t}_{3} - {\bf u}_2\sigma_2. \label{SMeq::assignment_tab}
\end{equation}
Therefore, 
$$R^* - R^{**}= M={\bf u}_1\sigma_1 {\bf v}_1^T + {\bf u}_2\sigma_2 {\bf v}_2^T = {\bf t}_3 {\bf a}^T+({\bf t}_3-{\bf t}_2){\bf b}^T.$$
Now we set $R_2$ to be some rotation  matrix and $R_3=R^{**}$. Clearly,  $${\bf t}_3 {\bf a}^T+({\bf t}_3-{\bf t}_2){\bf b}^T=R_2R_2^TR^{*}-R_3$$
which means
$$R_2^TR^*=R_{2}^T(R_{3}+{\bf t}_3 {\bf a}^T+({\bf t}_3-{\bf t}_2) {\bf b}^T),$$
yielding 
$$F_{23}=[R_2^T({\bf t}_2-{\bf t}_3)]_{\times} R_2^TR^{*},$$
which is an essential matrix that in general is inconsistent  with $F_{12},F_{13}$.

\subsection{Technical details of the code provided for demonstrating a counter example}
We first randomly sampled
two rotation matrices $R^{*},R^{**}$ until we obtained $M=R^*-R^{**}$ of rank $2$, and we set $R_3 = R^{**}$.  Then, we  sampled randomly a rotation matrix, and we set  $R_2$. These selections are stored in a file called ``counter\_data.mat". 

In the code we  assign values to ${\bf t}_2, {\bf t}_3, R_3, {\bf a}, {\bf b}$, according to \eqref{SMeq::assignment_tab}. We verify that indeed $F_{12},F_{13}$ and $F_{23}$ are essential matrices. By construction, the $3$-view fundamental matrix $F$ is consistent.

In order to verify that the essential matrices are not consistent we extract the relative rotations from them.  Each essential matrix defines two possible relative rotations. We evaluate the relation $R_{12}R_{23}R_{31}$ for each of  the 8 choices of  triplet of relative rotations, and verify that none of them  closes a loop, i.e the following always holds $$R_{12}R_{23}R_{31}\neq I.$$   
{\bf An interesting observation.} In addition, we  verify in the code  that for any choice of signs for the eigenvectors of $F$, $X,Y$,  it turns out that $\sqrt{0.5}(X+Y)$ is indeed  not a block rotation matrix.  Interestingly, it holds that $\Sigma_+=-\Sigma_-$, which shows that by itself this condition, in this case, is not sufficient for  defining an appropriate consistent essential matrix. To conclude, requiring from a set of essential matrices to fulfill the sufficient conditions for consistent $n$-view fundamental matrix does not provide sufficient conditions for generating from this set,  a consistent essential matrix. The counter example demonstrates the inconsistency of the essential matrices, as well as   the violation of  our  conditions on the consistency of an  $n$-view essential matrix.    

\section{Proof of Lemma 5}
Below we prove Lemma 5 from the paper (which we rename here to be Lemma 6). 

\begin{lemma}\label{SMSVDtoSPEC} Let $E\in \mathbb{S}^{3n}$ of rank(6), and $\Sigma \in \Real^{3 \times 3} $, a diagonal matrix, with positive elements on the diagonal. Let $X, Y, U, V \in \Real^{3n \times 3}$, and we define the mapping $(X,Y) \leftrightarrow (U,V):$ $X = \sqrt{0.5}({\hat U}+{\hat V})$, $Y = \sqrt{0.5}({\hat V}-{\hat U})$, ${\hat U} = \sqrt{0.5}(X-Y), {\hat V} = \sqrt{0.5}(X+Y)$. 

Then, the (thin) SVD of $E$ is of the form $$E=\left[ {\hat{U}},{\hat{V}}\right]\left(\begin{array}{cc}
\Sigma\\
 & \Sigma
\end{array}\right)\left[\begin{array}{c}
 {\hat{V}}^{T}\\
 {\hat{U}}^{T}
\end{array}\right]$$   if and only if the (thin) spectral decomposition of $E$ is 
of the form 

$$E=[X, Y]\left(\begin{array}{cc}
\Sigma\\
 & -\Sigma
\end{array}\right)\left[\begin{array}{c}
 X^{T}\\
 Y^{T}
\end{array}\right]$$ 
\end{lemma}
\begin{proof}
($\Rightarrow$)      
\begin{align}\label{SMeq:SVD2Spec}
E &=\left[ \hat{U},\hat{V}\right]\left(\begin{array}{cc}
\Sigma\\
 & \Sigma
\end{array}\right)\left[\begin{array}{c}
 \hat{V}^{T}\\
 \hat{U}^{T}
\end{array}\right]=\hat{U}\Sigma{\hat{V}^{T}}+\hat{V}\Sigma{\hat{U}^{T}}=\nonumber \\
 & 0.5\cdot(\hat{U}+\hat{V)}\Sigma(\hat{U}+\hat{V})^{T} -0.5\cdot(\hat{V}-\hat{U)}\Sigma(\hat{V}-\hat{U})^{T}=\nonumber \\
 & 0.5\left[\hat{U}+\hat{V},\hat{V}-\hat{U}\right]\left(\begin{array}{cc}
\Sigma\\
 & -\Sigma
\end{array}\right)\left[\hat{U}+\hat{V},\hat{V}-\hat{U}\right]^{T} =\nonumber \\ 
&
[X,Y]\left(\begin{array}{cc}
\Sigma\\
 & -\Sigma
\end{array}\right)\left[\begin{array}{c}
 X^{T}\\
 Y^{T}
\end{array}\right]
\end{align}
where  $X = \sqrt{0.5}({\hat U}+{\hat V}) $ and $ Y = \sqrt{0.5}({\hat V}-{\hat U})$. Since, $ \left[\begin{array}{c}
 \hat{U}^{T}\\
 \hat{V}^{T}
\end{array}\right]\left[ \hat{U},\hat{V}\right]=I_{6 \times 6}$, it yields $ \left[\begin{array}{c}
X^{T}\\
Y^{T}
\end{array}\right]\left[ X,Y\right]=I_{6 \times 6}$, concluding that the last term in \eqref{SMeq:SVD2Spec} is indeed (thin) spectral decomposition of $E$.

 ($\Leftarrow$) 
\begin{align}\label{SMeq:spec2SVD}
E&=[X,Y]\left(\begin{array}{cc}
\Sigma\\
 & -\Sigma
\end{array}\right)[X,Y]^{T} = X\Sigma X^T-Y\Sigma Y^T= \nonumber \\
&0.5(X+Y)\Sigma(X-Y)^T+0.5(X-Y)\Sigma(X+Y)^T= \nonumber\\ &0.5[X-Y,X+Y]\left(\begin{array}{cc}
\Sigma\\
 & \Sigma
\end{array}\right)[X+Y,X-Y]^T = \nonumber \\
&[{\hat U},{\hat V}]\left(\begin{array}{cc}
\Sigma\\
 & \Sigma
\end{array}\right)\left[\begin{array}{c}
 {\hat{V}}^{T}\\
 {\hat{U}}^{T}
\end{array}\right]
\end{align}
where  ${\hat U} = \sqrt{0.5}(X-Y) $ and $ {\hat V} = \sqrt{0.5}(X+Y)$. The same argument for orthogonality works here, showing that indeed that last term in \eqref{SMeq:spec2SVD} is SVD of E.
\end{proof}


%
\section{Handling scaled rotations}
Our optimization enforces the consistency of camera triplets, while allowing the essential matrices to be scaled arbitrarily. This is possible, because our theory can be generalized to handle scaled rotations. Below we generalize  Definition 4 from the paper to allow essential matrices of the form $E_{ij}= \alpha_i R_{i}^{T}([\mathbf{t}_{i}]_{\times}-[\mathbf{t}_{j}]_{\times})R_{j}\alpha_j$ and prove that the main theorem, i.e., Theorem 3 from the paper, holds for this generalization as well. This argument is then used to justify our treatment of camera triplets.    
\begin{definition}
 \label{SMdef:consistent_mv_essential}
 An $n$-view essential  matrix $E$ is called {\bf scaled consistent}  if there exist $n$  rotation matrices $\{ R_i\}_{i=1}^{n}$, $n$  vectors $\{{\bf t}_i\}_{i=1}^{n}$ and $n$ non-zero scalars  $\{ \alpha_i \}_{i=1}^n$  such that $E_{ij}= \alpha_i R_{i}^{T}([\mathbf{t}_{i}]_{\times}-[\mathbf{t}_{j}]_{\times})R_{j}\alpha_j$. \end{definition}

\noindent The following theorem is a generalized version of Thm.~2 from the paper, and the derivations here are inspired by the derivations made in [19]. 
\begin{theorem}\label{SMthm:necessary_essential} Let $E$ be a scaled consistent $n$-view essential matrix, associated with scaled rotation matrices $\{\alpha_i R_i\}_{i=1}^n$, $\alpha_i\neq0$ and camera centers $\{{\bf t}_i\}_{i=1}^n$.  $E$ satisfies the following conditions
\begin{enumerate}
\item $E$ can be formulated as $E = A + A^T$ where $A=UV^T$ and $U, V \in \Real^{3n \times 3}$
\begin{align*}
V =\left[\begin{array}{ccc}
\alpha_1 R_{1}^T\\
\vdots\\
\alpha_n R_{n}^T
\end{array}\right] ~~ & ~~
U =\left[\begin{array}{c}
\alpha_1 R_{1}^T T_{1}\\
\vdots\\
\alpha_n R_{n}^T T_{n}
\end{array}\right]
\end{align*}
with $T_i = [{\bf t}_i]_{\times}$ and w.l.o.g $\sum_{i=1}^n \alpha_i^2{\bf   t}_i=0$.

\item Each column of $U$ is orthogonal to each column of $V$, i.e., $V^T U=0_{3 \times 3}$

\item rank(V)=3 

\item If not all $\{{\bf t}_i\}_{i=1}^n$ are collinear, then rank(U) and rank(A) = 3. Moreover, if the (thin) SVD of $A$ is  $A=\hat{U}\Sigma\hat{V^{T}}$, with  $\hat {U}, \hat {V} \in \Real^{3n \times 3}$ and $\Sigma \in \Real^{3 \times 3}$ then
 the (thin) SVD of $E$ is $$E = \left[\hat{U},\hat{V}\right]\left(\begin{array}{cc}
\Sigma\\
 & \Sigma
\end{array}\right)\left[\begin{array}{c}
\hat{V}^{T}\\
\hat{U}^{T}
\end{array}\right]$$
implying  rank(E) = 6.
\end{enumerate}
\end{theorem}

\begin{proof}

\begin{enumerate}
\item The decomposition is a straightforward result from Def. \ref{SMdef:consistent_mv_essential}. Moreover, any global translation of all the camera centers, will not change the values of the entries of $E$.  In particular, if we denote the camera centers by  $\{{\bf \tilde{t}}_i\}_{i=1}^n$ and they are  translated to their new position ${\bf t}_i = {\bf \tilde{t}}_i - \frac{\sum \alpha_i^2 {\bf \tilde{t}}_i }{\sum \alpha_i^2}  $, then $\sum \alpha_i^2 {\bf t}_i =0$.    

\item By the decomposition above, we have that 
$$V^TU=\sum_{i=1}^{i=n}\alpha_i^2R_iR_i^TT_i=[\sum_{i=1}^{i=n}\alpha_i^2{\bf t}_i]_\times=0$$
which concludes that each column of $U$ is orthogonal to each column of $V$.\\
\item $rank(V)=3$ since each block of $V$ is of rank 3 and $V$ has 3 columns.\\
\item  Assume by contradiction that $rank(U)< 3$. Then, $\exists {\bf t} \in \Real^3, {\bf t} \neq 0$,
 s.t. $U {\bf t} = 0$. This implies that ${\bf t}_{i} \times {\bf t} = 0$ for all
$i=1,\ldots,n$. This implies that all the ${\bf t}_i$'s are parallel to ${\bf t}$,  violating our assumption that not all camera locations are collinear. Consequently $rank(U) = 3$ and therefore also $rank(A) = 3$.
Finally, let $A=\hat{U}\Sigma \hat{V}^T$ the SVD of $A$. Since $A=UV^T$ we get that 
$$Span(U)=Span(\hat{U}),Span(V)=Span(\hat{V}).$$ 
Then, since $E=A+A^T$, we get 
$$E=\hat{U}\Sigma \hat{V}^T+\hat{V}\Sigma \hat{U}^T=  \begin{bmatrix}\hat{U} &\hat{V} \end{bmatrix}\begin{pmatrix}\Sigma & 0 \\
0 & \Sigma \\
\end{pmatrix} \begin{bmatrix} \hat{V}^T \\ \hat{U}^T \end{bmatrix}.$$ 

Following the result that the columns of $U$ are orthogonal to those of $V$, it turns out that  $\begin{bmatrix} \hat{U} &\hat{V} \end{bmatrix}$ is column orthogonal,  concluding  that the form above is the SVD of $E$, and $rank(E)=6$.
\end{enumerate}
\end{proof}

Next, we show that Thm. 3 in the paper is also applicable with the generalized definition, i.e., Def. \ref{SMdef:consistent_mv_essential}.

\begin{theorem}
Let $E \in \mathbb{S}^{3n}$ be a  consistent $n$-view fundamental matrix with a set of $n$ cameras  whose centers are not all collinear. We denote by $\Sigma_+ ,\Sigma_- \in \mathbb{R}^{3\times3}$ the diagonal matrices with   the 3 positive and 3 negative eigenvalues of $E$, respectively. The following conditions are equivalent:
\begin{enumerate}
\item $E$ is a scaled consistent $n$-view essential matrix
\item The (thin) SVD of $E$   can be written in the form  $$E=\left[ {\hat{U}},{\hat{V}}\right]\left(\begin{array}{cc}
\Sigma_{+}\\
 & \Sigma_{+}
\end{array}\right)\left[\begin{array}{c}
 {\hat{V}}^{T}\\
 {\hat{U}}^{T}
\end{array}\right]$$
with $\hat{U}, \hat{V} \in \Real^{3n \times 3}$ such that each $3 \times 3$ block of  $\hat{V}$, $\hat{V}_i$, $i=1, \ldots, n$,  is a scaled rotation matrix, i.e., ${\hat V_i}=\hat {\alpha}_i {\hat R}_i $, where ${\hat R}_i \in SO(3) $ and ${\hat \alpha}_i \neq 0$.   
We say that $\hat{V}$ is a scaled block rotation matrix.

\item $\Sigma_+ = - \Sigma_-$ and the  (thin) spectral decomposition of $E$ is of the form $$E=[X, Y]\left(\begin{array}{cc}
\Sigma_{+}\\
 & \Sigma_{-}
\end{array}\right)\left[\begin{array}{c}
 X^{T}\\
 Y^{T}
\end{array}\right]$$   such that $\sqrt{0.5}(X+Y)$ is a scaled block rotation matrix.
\end{enumerate}
\label{SMconsistencyThm} 
\end{theorem} 
\begin{proof}
\textit{(1)$\Rightarrow$(2)}  Assume that $E$ is a scaled consistent $n$-view essential matrix. Then, according to Thm.~\ref{SMthm:necessary_essential}, $E=A+A^T$ with $A=UV^T$ and $U, V \in \Real^{3n \times 3}$ with 
\begin{align*}
V =\left[\begin{array}{ccc}
\alpha_1 R_{1}^T\\
\vdots\\
\alpha_n R_{n}^T
\end{array}\right] ~~ & ~~
U =\left[\begin{array}{c}
\alpha_1 R_{1}^T T_{1}\\
\vdots\\
\alpha_n R_{n}^T T_{n}
\end{array}\right]
\end{align*}
where $T_i = [{\bf t}_i]_{\times}$. 
Since $A=UV^T$ and $rank(A)=3$, then $A^TA=VU^TUV$ and $A^TA \succeq 0$ with $rank(A^TA)=3$ ($A$ and $A^TA$ share the same null space). First, we construct  a spectral decomposition to $A^TA$, relying on the special properties of $U$ and $V$. We have $rank(U)=3$, and therefore $U^TU$, which is a  $3 \times 3$, symmetric positive semi-definite matrix, is of full rank. Its spectral decomposition is of the form $U^TU = P D P^T$, where $P \in SO(3)$. (Spectral decomposition guarantees that $P \in O(3)$. However, $P$ can be replaced by $-P$ if $\det(P)=-1$.) $D \in \Real^{3 \times 3}$ is a diagonal matrix consisting of the (positive) eigenvalues  of $U^TU$. This spectral decomposition yields the following decomposition 
\begin{equation}\label{SMeq:decomposition}
A^T A = VPDP^TV^T.
\end{equation}
Now, note that
\begin{align*}
 P^TV^T VP = P^T \begin{bmatrix}\alpha_1 R_1 & \ldots & \alpha_nR_n \end{bmatrix} \begin{bmatrix}\alpha_1R_1^T \\ \vdots \\\alpha_n R_n^T \end{bmatrix} P\\ = P^T(\sum_{i=1}^n \alpha_i^2)I_{3 \times 3} P= (\sum_{i=1}^n \alpha_i^2)I_{3 \times 3}. 
\end{align*}
Let $\alpha = \sum_{i=1}^n \alpha_i^2$.
By a simple manipulation \eqref{SMeq:decomposition} becomes a spectral decomposition  
\begin{equation}\label{SMeq:spectral1}
 A^TA= \frac{1}{\sqrt{\alpha}}VP(\alpha D)P^TV^T\frac{1}{\sqrt{\alpha}}.\end{equation}
On the other hand, the (thin) SVD of $A$ is of the form $A ={\hat U} \Sigma {\hat V}^T$, where ${\hat U}, {\hat V} \in \Real^{3n \times 3}$, $\Sigma \in \Real^{3 \times 3}$. This means that 
\begin{equation}\label{SMeq:spectral2}
A^TA = {\hat V} \Sigma^2 {\hat V}^T.
\end{equation}

Due to the uniqueness of the eigenvector decomposition, \eqref{SMeq:spectral1} and \eqref{SMeq:spectral2} collapse to the same eigenvector decomposition, up to some global rotation,  $H \in SO(3)$, that is  $\frac{1}{\sqrt{\alpha}}VP= {\hat V} H$, which means that 

\begin{align} \label{SMroteq}
{\hat V}_i = \frac{\alpha_i}{\sqrt{\alpha}} R_i^TP H^T.
\end{align}
 Since $R_i^T, P, H^{T} \in SO(3)$, then setting $\hat {\alpha}_i =:\frac{\alpha_i}{\sqrt{\alpha}} $ and  ${\hat R}_i=:R_i^TPH^{T} \in SO(3)$, shows that ${\hat V}$ is a scaled block rotation matrix. Finally, by Thm. \ref{SMthm:necessary_essential}, the (thin) SVD of $E$ is of the form 

\begin{equation}\label{SMeq:E_SVD}
E = \left[\hat{U},\hat{V}\right]\left(\begin{array}{cc}
\Sigma\\
 & \Sigma
\end{array}\right)\left[\begin{array}{c}
\hat{V}^{T}\\
\hat{U}^{T}
\end{array}\right]
\end{equation}
and according to Lemma \ref{SMSVDtoSPEC}, 
the eigenvalues of $E$ are $\Sigma$ and $-\Sigma$. Since the elements on the diagonal of $\Sigma$ are positive, and $E$ is symmetric with exactly 3 positive eigenvalues $\Sigma_+$ and 3 negative eigenvalues $\Sigma_-$,  it follows that $\Sigma = \Sigma_+$ and $-\Sigma = \Sigma_-$ concluding the proof.

\textit{(2)$\Rightarrow$(1)} Let $E$ be a consistent $n$-view fundamental matrix that satisfies condition (2). We would like to show that $E$  is a scaled consistent $n$-view essential matrix. By condition (2) $E$ can be written as
\begin{equation}\label{SMeq:E_hat}
E={\hat{U}}\Sigma_+{\hat{V}^{T}}+\hat{V}\Sigma_+{\hat{U}^{T}}=\bar{{U}}\hat{V}^T+\hat{V}\bar{U^{T}}
\end{equation}
where ${\bar U}={\hat U} \Sigma_+$ with  ${\hat V}_i = \hat{\alpha_i} {\hat R}_i$, ${\hat R}_i \in SO(3)$. By definition  $E_{ii}=0$, and this implies that $\bar{U}_i\hat{V}_i^T$ is a skew symmetric matrix. Using  Lemma 4 in the paper,  $\bar{U}_i=\hat{V}_i {\hat T}_i $ for some skew symmetric matrix ${\hat T}_i = [{\hat {\bf t}}_i]_{\times}$. Plugging ${\bar U}_i$ and ${\hat V}_i$ in \eqref{SMeq:E_hat} yields 
\begin{align}E_{ij}=\bar{U}_i\hat{V}_j^T+\hat{V}_i\bar{U}_j^T= \hat{\alpha_i}\hat{\alpha_j}{\hat R}_i {\hat T}_i {\hat R}_j^T - \hat{\alpha_i}\hat{\alpha_j}{\hat R}_i {\hat T}_j {\hat R}_j^T \nonumber\\= {\hat{\alpha_i}R_i}^T([{\bf t}_i]_{\times} - [{\bf t}_j]_{\times})R_j\hat{\alpha_j}\nonumber\end{align}
where $R_i = {\hat R}_i^T$, $\alpha_i = \hat{\alpha}_i$ and  ${\bf t}_i =  {\hat{\bf t}}_i$, concluding the proof. Finally, the derivation of the equivalence 
$(2)\Leftrightarrow(3)$ is exactly as in Thm. 3 in the paper. 
\end{proof}

\begin{corollary}
Let $E$ be a scaled consistent $n$-view matrix, then 
\begin{enumerate}
\item The scale of each block ${\hat V}_i$ in the scaled block rotation matrix ${\hat V}$ can be calculated using by $\hat{\alpha_i}=({det({\hat V}_i)})^{\frac{1}{3}}$
\item Let $E$ be a scaled consistent $n$-view essential matrix. Then, the transformation  $diag(\frac{1}{\hat{\alpha_1}} I_{3\times3},\ldots,\frac{1}{\hat{\alpha_n}} I_{3\times3})\cdot E \cdot diag(\frac{1}{\hat{\alpha_1}} I_{3\times3},\ldots,\frac{1}{\hat{\alpha_n}} I_{3\times3})$ transform $E$ to be a consistent $n$-view essential matrix.
\end{enumerate}
\end{corollary}
\begin{corollary}\label{SMcorollary::scale_factors3} A scaled consistent 3-view essential matrix is invariant to pairwise scaling.
\end{corollary}
\begin{proof}
This proof is inspired by the proof of Corollary 2, presented in [15].
Let $E$ be a scaled consistent 3-view essential matrix whose blocks are defined as $E_{ij}=\alpha_i R_i^T(T_i-T_j)R_j \alpha_j$, and let $\widetilde{E}$ be a $9 \times 9$ matrix whose blocks are defined to be $\widetilde{E}_{ij}=s_{ij}E_{ij}$ where $s_{ij} \ne 0$ are arbitrary pairwise scale factors.
Without loss of generality we can assume that the number of negative scale factors is even (otherwise we can multiply the entire matrix by -1). Therefore, $s_1=(\frac{ s_{12}s_{13}}{s_{23}})^{\frac{1}{2}}$,
$s_2=(\frac{ s_{23}s_{12}}{s_{13}})^{\frac{1}{2}}$, and
$s_3 =(\frac{ s_{13}s_{23}}{s_{12}})^{\frac{1}{2}}$ determine real values such
that $s_1s_{2}=s_{12}$, $s_1s_{3}=s_{13}$, and $s_2s_{3}=s_{23}$.
Let $\widetilde{\alpha}_i=s_i \alpha_i$ for $i=1,2,3$, we get that
\begin{equation}
\widetilde{E}_{ij} = s_{ij} E_{ij}= {\tilde \alpha}_i R_i^T(T_i-T_j)R_j {\tilde \alpha}_j
\end{equation}
Hence, ${\tilde E}$ is a scaled consistent 3-view essential matrix.
\end{proof}
\section{Uniqueness of consistent camera matrices}
This part deals with the argument stated in Corollary 1 in the paper, claiming that the  recovery of camera matrices from a consistent multiview essential matrix is unique up to a global similarity transformation. More formally, we claim 
\begin{theorem}\label{SMThm:unique}
Let E be a consistent n-view essential matrix and let $P_1, \ldots, P_n$ be a set of  non-collinear camera matrices which  is consistent with $E$, then this set of camera matrices  is unique up to  some global similarity transformation.
\end{theorem}

We first justify the argument for 3 cameras, exemplify its extension for 4 cameras and finally derives an induction for $n$ cameras.
 
A calibrated camera matrix $P_{i}$ is represented as $$P_i(R_i,{\bf t}_i) = [R_i^T | -R_i^Tt_i]\in \mathbb{R}^{3\times 4} $$ 
where $R_i$ and ${\bf t}_i$ are the orientation and location of the i'th camera, respectively. In addition, we represent a similarity transformation $S(s,R,{\bf t})$ by a matrix of the form  
$$S=\begin{bmatrix}
sR & t \\
0^T & 1 
\end{bmatrix}\in \mathbb{R}^{4\times 4}.$$
Then, applying a similarity transformation on a camera matrix $P_i$ yields 
$$S(P_i)=P S =[sR_i^T R| R_i^T(t-t_i)]$$ which is  
$$\left[R_i^T R \vert \frac{R_i^T(t-t_i)}{s}\right]$$
in a calibrated format.

\noindent
{\bf 3 cameras.} We consider the case of 3 non-collinear cameras. Note that every point
correspondence between two of the views can be extended
to the third view by intersecting epipolar lines. Consequently,
we can produce any number of correspondences
across the three views. Uniqueness of reconstruction then
follows from [12]  who
proved that 4 points in three views generally yield unique
camera recovery.

\noindent
{\bf 4 cameras.}  We consider a consistent  4-view essential matrix $E$ with 4 corresponding cameras  
$$P_1(R_1,{\bf t}_1), P_2(R_2,{\bf t}_2), P_3(R_3,{\bf t}_3), P_4(R_4,{\bf t}_4).$$ 
Suppose we have another set of cameras which is consistent with $E$,
$$P_1^*(R_1^*,{\bf t}_1^*), P_2^*(R_2^*,{\bf t}_2^*), P_3^*(R_3^*,{\bf t}_3^*), P_4^*(R_4^*,{\bf t}_4^*).$$
We show that these two sets are the same up to a global similarity transformation. Without loss of generality we assume that the third camera is not collinear with the rest of cameras. From the consistency of $E_{123}$, there exists a similarity transformation ${\bar S}$ such that 
$$ {\bar S}(P_1) = P_1^*,  {\bar S}(P_2) = P_2^*, {\bar S}(P_3) = P_3^*$$
and from the consistency of $E_{234}$  there exists a similarity transformation ${\hat S}$ such that  
$$ {\hat S}(P_2) = P_2^*,  {\hat S}(P_3) = P_3^*, {\hat S}(P_4) = P_4^*.$$
Now, since $P_2^* = {\bar S}(P_2)= {\hat S}(P_2)$, it yields 
$$\left[R_2^T \bar{R}| \frac{R_2^T(\bar{t}-t_2)}{\bar{s}} \right]= \left[R_2^T \hat{R}| \frac{R_2^T(\hat{t}-t_2)}{\hat{s}} \right]$$
and consequently 
$$R_2^T \bar{R}=R_2^T \hat{R}\Rightarrow \bar{R}=\hat{R}.$$
Then we have $$\frac{R_2^T(\bar{t}-t_2)}{\bar{s}}=\frac{R_2^T(\hat{t}-t_2)}{\hat{s}}\Rightarrow \frac{\bar{t}}{\bar{s}}-\frac{t_2}{\bar{s}}=\frac{\hat{t}}{\hat{s}}-\frac{t_2}{\hat{s}}$$
and similarly from the relation $P_3^* = {\bar S}(P_3)= {\hat S}(P_3)$ we get that $$\frac{R_3^T(\bar{t}-t_2)}{\bar{s}}=\frac{R_3^T(\hat{t}-t_2)}{\hat{s}}\Rightarrow \frac{\bar{t}}{\bar{s}}-\frac{t_3}{\bar{s}}=\frac{\hat{t}}{\hat{s}}-\frac{t_3}{\hat{s}}.$$
By subtracting the last two equations we have $$\frac{t_3-t_2}{\bar{s}}=\frac{t_3-t_2}{\hat{s}}\Rightarrow \bar{s}=\hat{s}\Rightarrow\bar{t}=\hat{t},$$
implying that ${\bar S} = {\hat S}$. 
\begin{proof}
The proof is by induction. The argument was proved for 3 cameras and now we suppose that we have a consistent $n$-view essential matrix. From the induction assumption we know that since  $E_{1, \ldots, (n-1)}$ is consistent, the recovery of the camera matrices is unique and similarly the recovery of the camera matrices from $E_{2, \ldots, n}$ is unique. As in 4 cameras we assume  without loss of generality that the third camera is not collinear with the rest of the cameras, and a similar derivation yields that the recovery of the cameras is unique up to some global similarity transformation.
\end{proof}


        
\end{document}